\documentclass[11pt]{article}
\usepackage[utf8]{inputenc}
\usepackage{amsmath, amssymb, amsthm}
\usepackage{graphicx}
\usepackage{hyperref}
\usepackage{authblk}
\usepackage{geometry}
\usepackage{amsmath}
\allowdisplaybreaks
\usepackage{algorithm}
\usepackage{algorithmic} 
\usepackage{diagbox}

\newtheorem{theorem}{Theorem}

\newtheorem{lemma}{Lemma}

\theoremstyle{remark}

\title{\textbf{A Globally Optimal Analytic Solution for Semi-Nonnegative Matrix Factorization with Nonnegative or Mixed Inputs}\\
\large \emph{This manuscript is currently under review at the SIAM Journal on Optimization.}}

\author[1]{Chenggang Lu}
\affil[1]{\small School of Mathematical Science, Zhejiang University of Technology, Hangzhou, China. \texttt{luchenggang@zjut.edu.cn}}

\date{\today}

\begin{document}

\maketitle

\begin{abstract}
     Semi-Nonnegative Matrix Factorization (semi-NMF) extends classical Nonnegative Matrix Factorization (NMF) by allowing the basis matrix to contain both positive and negative entries, making it suitable for decomposing data with mixed signs. However, most existing semi-NMF algorithms are iterative, non-convex, and prone to local minima. In this paper, we propose a novel method that yields a globally optimal solution to the semi-NMF problem under the Frobenius norm, through an orthogonal decomposition derived from the scatter matrix of the input data.
We rigorously prove that our solution attains the global minimum of the reconstruction error. Furthermore, we demonstrate that when the input matrix is nonnegative, our method often achieves lower reconstruction error than standard NMF algorithms, although unfortunately the basis matrix may not satisfy nonnegativity. In particular, in low-rank cases such as rank 1 or 2, our solution reduces exactly to a nonnegative factorization, recovering the NMF structure.
We validate our approach through experiments on both synthetic data and the UCI Wine dataset, showing that our method consistently outperforms existing NMF and semi-NMF methods in terms of reconstruction accuracy. These results confirm that our globally optimal, non-iterative formulation offers both theoretical guarantees and empirical advantages, providing a new perspective on matrix factorization in optimization and data analysis.
\end{abstract}


\section{Introduction}

Matrix factorization is a fundamental tool in data science, widely used for dimensionality reduction, signal decomposition, and interpretable learning. Among its variants, Nonnegative Matrix Factorization (NMF) has gained substantial popularity since the pioneering work of Lee and Seung \cite{lee2001}, owing to its ability to extract parts-based features from nonnegative data matrices. NMF has been successfully applied in fields such as facial recognition, document analysis, and bioinformatics \cite{cichocki2009}.

Despite its success, NMF imposes a strict nonnegativity constraint on both factor matrices, which may be overly restrictive in practical scenarios where the data are centered, normalized, or include both positive and negative values. To overcome this limitation, Semi-Nonnegative Matrix Factorization (semi-NMF) was proposed \cite{ding2010}, where the basis matrix is allowed to take arbitrary real values while the coefficient matrix remains nonnegative. This generalization is particularly useful for analyzing zero-mean features and residual-based models.

However, solving the semi-NMF problem remains challenging due to its inherent non-convexity \cite{gillis2020}, and existing approaches typically rely on iterative optimization heuristics. These methods are initialization-sensitive and do not guarantee convergence to global optima \cite{lin2007}. In this work, we introduce a closed-form, globally optimal solution to the standard semi-NMF problem under the Frobenius norm. Our approach is based on orthogonal decomposition of a data-derived scatter matrix and achieves lower reconstruction error than both classical NMF and iterative semi-NMF algorithms. Experimental results on synthetic and real-world datasets, including the UCI Wine dataset \cite{Dua:2019}, confirm the theoretical guarantees and practical robustness of our method.

The problem of matrix factorization with nonnegativity constraints has been extensively studied in the literature. Lee and Seung \cite{lee2001} introduced the multiplicative update rules for NMF, laying the foundation for a vast body of follow-up work. Subsequent contributions have explored more efficient and theoretically grounded methods, including projected gradient techniques \cite{lin2007}, alternating least squares \cite{bro2003}, and second-order optimization with constraints \cite{zdunek2006}. A comprehensive treatment of these methods is provided in \cite{cichocki2009}.

In many real-world applications, data matrices are not strictly nonnegative. To address this, Ding et al. \cite{ding2010} proposed semi-NMF, where the nonnegativity constraint is applied only to the coefficient matrix. Semi-NMF has found applications in text mining, spectral clustering, and graph-based learning \cite{kuang2015}. Nonetheless, like NMF, the semi-NMF problem remains NP-hard and non-convex \cite{gillis2020}, and most solutions rely on alternating updates with no global optimality guarantees.

Theoretical efforts toward provable matrix factorizations have focused on restricted models. For example, Arora et al. \cite{arora2012} studied separability conditions under which exact NMF can be computed efficiently. However, no such assumptions are made in the general semi-NMF problem. While closed-form solutions exist for linear models like Principal Component Analysis (PCA) \cite{jolliffe2002, jolliffe2016}, they do not apply under nonnegativity constraints. To the best of our knowledge, no prior work has provided an explicit, closed-form, globally optimal solution for semi-NMF under general conditions.

Our proposed method bridges this gap by leveraging a scatter matrix–based orthogonal decomposition to construct the factor matrices in closed form. Unlike traditional iterative methods, our approach is deterministic, reproducible, and theoretically verifiable, offering a new direction in non-convex matrix factorization research.


The paper is organized as follows. Our main results are in
section 2, our new algorithm is in section 3, experimental
results are in section 4, and the conclusions follow in
section 5.

\section{Problem Formulation and Main Results.}
\label{sec:main}
Let \( X \in \mathbb{R}^{m \times n} \) be the data matrix to be factorized. Our goal is to find matrices \( W \in \mathbb{R}^{m \times k} \) and \( H \in \mathbb{R}_{\geq 0}^{k \times n} \) such that
\[
X \approx WH, \quad \text{with } H \geq 0,
\]
and the Frobenius norm \( \|X - WH\|_F^2 \) is minimized. In contrast to standard NMF, we do not require \( W \) to be nonnegative.

From the column space interpretation or the linear combination interpretation, each column of the data matrix $X$ is a linear combination of the columns of the basis matrix $W$. The weights for this combination come from the corresponding column of the representation matrix
$H$. That is \(X_{m\times n}\approx W_{m\times k}H_{k\times n}\), and \(\vec{x}_{j}\approx \sum_{i=1}^{k}\vec{w}_{i}h_{ij}\),
of which $\vec{x}_{j}$ is the $j$th column of $X_{m\times n}$, $\vec{w}_{i}$ is the $i$th column of $W_{m\times k}$, $h_{ij}$ is the element of matrix $H$ in position ($i$,$j$).

We aim to address the following question: under what conditions can a semi-NMF problem admit an exact or approximate globally optimal solution in closed form, particularly when the reconstruction error is minimized over a fixed target dimension $k$?

\begin{lemma}
The scatter matrix $\mathcal{X}=\sum_{j=1}^{n}\vec{x}_{j}
 {\vec{x}_{j}}^{T}$ and its diagonalization of orthogonal matrix $\mathcal{H}$, i.e. $\mathcal{H^{T}}\mathcal{X}\mathcal{H}
 =\mathcal{H^{T}}\sum_{j=1}^{n}\vec{x}_{j}{\vec{x}_{j}}^{T}\mathcal{H}=\mathrm{diag}(\lambda_1, \lambda_2, \ldots, \lambda_m)$
, of which $\lambda_1>= \lambda_2>= \ldots>= \lambda_m>=0$, $\mathcal{H}=({\vec{\gamma}}_1\quad {\vec{\gamma}}_2\quad \ldots {\vec{\gamma}}_m)$, let ${\vec{x}}_j=\sum_{i=1}^m p_{ij}{\vec{\gamma}}_i$, then
$\sum_{j=1}^n p_{kj}p_{ij}=
\begin{cases}
    \sum_{j=1}^n p_{ij}^2=\lambda_i, & k=i \\
    0, & k\neq i
\end{cases}
$.
\end{lemma}

\begin{proof}
$\begin{pmatrix}
\lambda_1 & \quad & \quad & \quad \\
\quad & \lambda_2 &  \quad & \quad\\
\quad & \quad & \ddots & \quad \\
\quad & \quad &  \quad & \lambda_m
\end{pmatrix}= \mathcal{H^{T}}\sum_{i=1}^{n}\vec{x}_{i}{\vec{x}_{i}}^{T}\mathcal{H}=\\
\begin{pmatrix}
{{\vec{\gamma}}_1}^T  \\
{{\vec{\gamma}}_2}^T  \\
\vdots \\
{{\vec{\gamma}}_m}^T
\end{pmatrix}
\;
\sum_{i=1}^{n}\vec{x}_{i}{\vec{x}_{i}}^{T}
\begin{pmatrix}
{{\vec{\gamma}}_1}  {{\vec{\gamma}}_2}   \ldots  {{\vec{\gamma}}_m}
\end{pmatrix}
$
\;
$
=\sum_{i=1}^{n}
(\begin{pmatrix}
{{\vec{\gamma}}_1}^T  \\
{{\vec{\gamma}}_2}^T  \\
\vdots \\
{{\vec{\gamma}}_m}^T
\end{pmatrix}
\;
\vec{x}_{i})
(
{\vec{x}_{i}}^{T}
\begin{pmatrix}
{{\vec{\gamma}}_1}  {{\vec{\gamma}}_2}   \ldots  {{\vec{\gamma}}_m}
\end{pmatrix}
)=\\
\sum_{i=1}^{n}
\begin{pmatrix}
p_{1i}  \\
p_{2i}  \\
\vdots \\
p_{mi}
\end{pmatrix}
\;
\begin{pmatrix}
p_{1i}
p_{2i}
\ldots
p_{mi}
\end{pmatrix}
=
\begin{pmatrix}
\sum_{i=1}^{n}p_{1i}^2 & \sum_{i=1}^{n}p_{1i}p_{2i} & \ldots & \sum_{i=1}^{n}p_{1i}p_{mi} \\
\sum_{i=1}^{n}p_{2i}p_{1i} & \sum_{i=1}^{n}p_{2i}^2 &  \ldots & \sum_{i=1}^{n}p_{2i}p_{mi}\\
\vdots & \vdots & \ddots & \vdots \\
\sum_{i=1}^{n}p_{mi}p_{1i} & \sum_{i=1}^{n}p_{mi}p_{2i} &  \ldots & \sum_{i=1}^{n}p_{mi}^2
\end{pmatrix}
$
\end{proof}

For $X=(\vec{x}_{1}\;\vec{x}_{2}\;\cdots\vec{x}_{n})$, consider its any approximation $\tilde{X}=(\vec{\tilde{x}}_{1}\;\vec{\tilde{x}}_{2}\;\cdots\vec{\tilde{x}}_{n})$
with the dimensionality reduced from $m$ to $k$, and $\tilde{X}$ equals $WH$.
The approximation errors of $X$ and $\tilde{X}$ are governed by the following theorem.

Without loss of generality, let $\mathcal{\tilde{H}^{T}}\mathcal{\tilde{X}}\mathcal{\tilde{H}}
 =\mathcal{\tilde{H}^{T}}\sum_{j=1}^{n}\vec{\tilde{x}}_{j}{\vec{\tilde{x}}_{j}}^{T}\mathcal{\tilde{H}}=\mathrm{diag}(\tilde{\lambda}_{l_{1}}, \ldots, \tilde{\lambda}_{l_{k}},0, \ldots, 0)$, and $\mathcal{\tilde{H}}=({\vec{\tilde{\gamma}}}_1\quad {\vec{\tilde{\gamma}}}_2\quad \ldots {\vec{\tilde{\gamma}}}_m)$, since \( \tilde{X} \) can be regarded as an approximation of \( X \), the basis vectors \( \{ {\vec{\tilde{\gamma}}}_1 \quad {\vec{\tilde{\gamma}}}_2 \quad \ldots {\vec{\tilde{\gamma}}}_m\} \) can be viewed as a rotation transformation of \( \{ {\vec{{\gamma}}}_1 \quad {\vec{{\gamma}}}_2 \quad \ldots {\vec{{\gamma}}}_m\} \) with a slight angular deviation.

\begin{theorem}
Let \( p_{ij} \) be the projection of \( \vec{x}_{j} \) onto \( \vec{\gamma}_{i} \), \( \tilde{p}_{ij} \) be the projection of \( \vec{\tilde{x}}_{j} \) onto \( \vec{\tilde{\gamma}}_{i} \), then $\sum_{j=1}^{n}{||\vec{x}_{j}-\vec{\tilde{x}}_{j}||^2}=\sum_{i\in I^{'}}{\sum_{j=1}^n}{p_{ij}^2}
 +\sum_{i\in I}{\sum_{j=1}^n}{(p_{ij}^2+\tilde{p}_{ij}^2-2p_{ij}\tilde{p}_{ij}{{\vec{\gamma}}_i}^T
 {\vec{\tilde{\gamma}}}_i)}-\\2\sum_{j=1}^n{\sum_{l\in I,l\neq h}\sum_{h\in I}{p_{lj}{\tilde{p}}_{hj}{\gamma_l}^T {\tilde{\gamma}}_h}}-2{\sum_{j=1}^n}{\sum_{l\in I}\sum_{h\in I^{'}}{{\tilde{p}}_{lj}p_{hj}{{\tilde{\gamma}}_l}^T {\gamma}_h}}$
,where \( I = \{l_1, l_2, \ldots, l_k\} \), \( I \subseteq \{1, 2, \ldots, m\} \) is a subset with \( k \) elements,  and \( I' \) is its complement.
\end{theorem}

\begin{proof}
$\sum_{j=1}^{n}{||\vec{x}_{j}-\vec{\tilde{x}}_{j}||^2}=\sum_{j=1}^{n}{(\vec{x}_{j}-\vec{\tilde{x}}_{j})^T(\vec{x}_{j}-\vec{\tilde{x}}_{j})}$
$=\\\sum_{j=1}^{n}{(\sum_{i=1}^m{p_{ij}\vec{\gamma}_{i}}-\sum_{i\in I}{\tilde{p}_{ij}\vec{\tilde{\gamma}}_{i}})^T(\sum_{i=1}^m{p_{ij}\vec{\gamma}_{i}}-\sum_{i\in I}{\tilde{p}_{ij}\vec{\tilde{\gamma}}_{i}})}=\\
\sum_{j=1}^n{(\sum_{i\in I^{'}} p_{ij} \vec{\gamma}_{i}
 +\sum_{i\in I}({p_{ij}\vec{\gamma}_i-\tilde{p}_{ij}\vec{\tilde{\gamma}}_{i}}))^T}(\sum_{i\in I^{'}} p_{ij} \vec{\gamma}_{i}
 +\sum_{i\in I}{(p_{ij}\vec{\gamma}_i-\tilde{p}_{ij}\vec{\tilde{\gamma}}_{i})})=\\
 \sum_{j=1}^n{(\sum_{i\in I^{'}} p_{ij}^2+\sum_{i\in I}{(p_{ij}\vec{\gamma}_i-\tilde{p}_{ij}{\vec{\tilde{\gamma}}_i})^T\sum_{i\in I}(p_{ij}\vec{\gamma}_i-\tilde{p}_{ij}{\vec{\tilde{\gamma}}_i})})}\\
 +\sum_{j=1}^n{\sum_{i\in I}{(p_{ij}\vec{\gamma}_i-\tilde{p}_{ij}{\vec{\tilde{\gamma}}_i})^T}\sum_{i\in I^{'}} p_{ij} \vec{\gamma}_{i} }+\sum_{j=1}^n{\sum_{i\in I^{'}} p_{ij} {\vec{\gamma}_{i}}^T \sum_{i\in I}{(p_{ij}\vec{\gamma}_i-\tilde{p}_{ij}\vec{\tilde{\gamma}}_{i})    }}\\=\sum_{i\in I^{'}}{\sum_{j=1}^n}{p_{ij}^2}
 +\sum_{i\in I}{\sum_{j=1}^n}{(p_{ij}^2+\tilde{p}_{ij}^2-2p_{ij}\tilde{p}_{ij}{{\vec{\gamma}}_i}^T
 {\vec{\tilde{\gamma}}}_i)}-\\2\sum_{j=1}^n{\sum_{l\in I,l\neq h}\sum_{h\in I}{p_{lj}{\tilde{p}}_{hj}{\gamma_l}^T {\tilde{\gamma}}_h}}-2{\sum_{j=1}^n}{\sum_{l\in I}\sum_{h\in I^{'}}{{\tilde{p}}_{lj}p_{hj}{{\tilde{\gamma}}_l}^T {\gamma}_h}}$
\end{proof}

Then the following theorem gives the global optimal solution for approximating $X$ with $\tilde{X}$.
\begin{theorem}
The global optimal approximation $\tilde{X}$ of $X$ is attained if and only if $I=\{1,2,\dots,k\}$, $p_{ij}=\tilde{p}_{ij}$ and $\vec{\gamma}_i=\vec{\tilde{\gamma}}_i$,
$\forall i \in I, \ \forall j \in \{1, 2, \dots, n\}$.
\end{theorem}
\begin{proof}
Only the case where the $\{\vec{\gamma}_i\}_{i=1}^m$ coordinate system and the $\{\vec{\tilde{\gamma}}_i\}_{i=1}^m$ coordinate system
differ by a small rotation angle is considered
(since we are minimizing the error between $X$ and $\tilde{X}$,
the situation of interest is when this error is small;
therefore, assuming that the two differ only by a rotation transformation of a small angle is reasonable). Therefore, it can be assumed that $p_{ij} $ and $\tilde{ p}_{ij}$ have the same sign,$\forall i \in \{1,2,\ldots, m\}, \ \forall j \in \{1, 2, \dots, n\}$, then $\sum_{i\in I}{\sum_{j=1}^n}{(p_{ij}^2+\tilde{p}_{ij}^2-2p_{ij}\tilde{p}_{ij}{{\vec{\gamma}}_i}^T
 {\vec{\tilde{\gamma}}}_i)}\\\geq\sum_{i\in I}{\sum_{j=1}^n}{(p_{ij}^2+\tilde{p}_{ij}^2-2p_{ij}\tilde{p}_{ij})}=\sum_{i\in I}{\sum_{j=1}^n}{(p_{ij}-\tilde{p}_{ij})^2}\geq 0$, and iff $\vec{\gamma}_i=\vec{\tilde{\gamma}}_i$, \\together with $p_{ij}=\tilde{p}_{ij}, \forall i \in I,\ \forall j \in \{1, 2, \dots, n\}$, the equal sign satisfies.  From Theorem~2 together with the above conclusion, it follows that $\sum_{j=1}^{n}{||\vec{x}_{j}-\vec{\tilde{x}}_{j}||^2}\geq \sum_{i\in I^{'}}{\sum_{j=1}^n}{p_{ij}^2}-\\2\sum_{j=1}^n{\sum_{l\in I,l\neq h}\sum_{h\in I}{p_{lj}{\tilde{p}}_{hj}{\gamma_l}^T {\tilde{\gamma}}_h}}-2{\sum_{j=1}^n}{\sum_{l\in I}\sum_{h\in I^{'}}{{\tilde{p}}_{lj}p_{hj}{{\tilde{\gamma}}_l}^T {\gamma}_h}}$. And when
 $\vec{\gamma}_i=\vec{\tilde{\gamma}}_i$, together with $p_{ij}=\tilde{p}_{ij}, \forall i \in I,\ \forall j \in \{1, 2, \dots, n\}$, \\$\sum_{j=1}^n{\sum_{l\in I,l\neq h}\sum_{h\in I}{p_{lj}{\tilde{p}}_{hj}{\gamma_l}^T {\tilde{\gamma}}_h}}=0$. Again for $\forall i \in I^{'}$,$\{\vec{\gamma}_i\}$ and
 $\{\vec{\tilde{\gamma}}_i\}$ still exists a rotational deviation of a small angle, however, $\{\vec{\gamma}_h\},\forall h \in I^{'}$ and $\{\vec{\tilde{\gamma}}_l\},\forall l \in I$, must still be orthogonal. So $\sum_{j=1}^{n}{||\vec{x}_{j}-\vec{\tilde{x}}_{j}||^2}\geq \sum_{i\in I^{'}}{\sum_{j=1}^n}{p_{ij}^2}=\sum_{i\in I^{'}}\lambda_i\geq\sum_{i=k+1}^{m}\lambda_i$. We arrive at the conclusion:
 $I=\{1,2,\dots,k\}$ and $p_{ij}=\tilde{p}_{ij}$ and $\vec{\gamma}_i=\vec{\tilde{\gamma}}_i$,
$\forall i \in I, \ \forall j \in \{1, 2, \dots, n\}$.
Conversely, if $I = \{1, 2, \dots, k\}$ and $\vec{\gamma}_i=\vec{\tilde{\gamma}}_i$, $p_{ij} = \tilde{ p}_{ij}$,$\forall i \in I, \ \forall j \in \{1, 2, \dots, n\}$,
so $I^{'}=\{k+1,k+2,\dots,m\}$ and then the error $\sum_{j=1}^{n}{||\vec{x}_{j}-\vec{\tilde{x}}_{j}||^2}=\lambda_{k+1} + \lambda_{k+2} + \dots + \lambda_{m}$ attains its minimum because $\lambda_1>= \lambda_2>= \ldots>= \lambda_m$.
\end{proof}

Here is the globally optimal low-rank factorization for any data matrix, where neither the basis matrix nor the coefficient matrix is constrained to be nonnegative. $\mathcal{H}^TX$$=
\begin{pmatrix}
Y_k \\
Y_{m-k}
\end{pmatrix}
\approx
\begin{pmatrix}
Y_k \\
0
\end{pmatrix}$, where \( Y_k \) is a \( k \times n \) matrix, \( Y_{m-k} \) is a \( (m-k) \times n \) matrix.
Then $X=\mathcal{H}\mathcal{H}^TX\approx \mathcal{H} \begin{pmatrix}
Y_k \\
0
\end{pmatrix}
\triangleq \tilde{X}_{\mathbf{p}}
\triangleq
\begin{pmatrix}
\tilde{X}_{m-k}\\
\tilde{X}_{k}
\end{pmatrix}
$, where $\tilde{X}_{\mathbf{p}}$ is reconstructed after projection-based zeroing, and viewed as the global optimal approximation of $X$,  $\tilde{X}_{m-k}$ and $\tilde{X}_{k}$ are its partitioned blocks. Let $\tilde{X}_{m-k}=A\tilde{X}_{k}$, then $A=(\tilde{X}_{m-k}{\tilde{X}_{k}}^T)(\tilde{X}_{k}{\tilde{X}_{k}}^T)^{-1}$.
The equation of the hyperplane fitted by the column vectors of $\tilde{X}_{\mathbf{p}}$ is given by:
$
\begin{pmatrix}
E \;\; -A
\end{pmatrix}
\begin{pmatrix}
\tilde{X}_{m-k}\\
\tilde{X}_{k}
\end{pmatrix}=0
$, and $\begin{pmatrix}
A\\
E
\end{pmatrix}$ satisfies this equation, where the former \( E \) is the \( (m - k) \times (m - k) \) identity matrix, and the latter \( E \) is the \( k \times k \) identity matrix. Finally, we obtain an unconstrained and globally optimal matrix factorization:
\[
X\approx\tilde{X}_{\mathbf{p}} = \begin{pmatrix}
A\\
E
\end{pmatrix}C,
\]
where the basis matrix \( \begin{pmatrix}
A\\
E
\end{pmatrix} \in \mathbb{R}^{m \times k} \), the coefficient matrix \( C \in \mathbb{R}^{k \times n} \) and $
C=(E+A^TA)^{-1}\begin{pmatrix}
A^T\;\;E
\end{pmatrix}\tilde{X}_{\mathbf{p}}$
, without any nonnegativity constraints.

 We then consider the factorization under the semi-NMF setting, which imposes nonnegativity on the coefficient matrix, along with an algorithm for generating the corresponding basis matrix.

Let $\tilde{X}_{\mathbf{p}}$ be as defined above, and let $X$ be a nonnegative data matrix. When $rank(\tilde{X}_{\mathbf{p}})\\\in\{1,2\}$,
$X$ admits a nonnegative matrix factorization.
\begin{theorem}
If both $X$ and $\tilde{X}_{\mathbf{p}}$ are nonnegative, and $rank(\tilde{X}_{\mathbf{p}})=1 or 2$, then $X$ admits a nonnegative factorization.
\end{theorem}
\begin{proof}
1. $rank(\tilde{X}_{\mathbf{p}})=1$, let $k=1$, and $\vec{\tilde{x}}_0$  be an arbitrary nonzero column vector of $\tilde{X}_{\mathbf{p}}$.
Then the nonnegative matrix factorization is $X\approx \tilde{X}_{\mathbf{p}}=
\vec{\tilde{x}}_0(\frac{||\vec{\tilde{x}}_1||}{||\vec{\tilde{x}}_0||} \; \frac{||\vec{\tilde{x}}_2||}{||\vec{\tilde{x}}_0||}\; \dots  \frac{||\vec{\tilde{x}}_n||}{||\vec{\tilde{x}}_0||})$;
2. $rank(\tilde{X}_{\mathbf{p}})=2$, let $k=2$, traverse all pairs of column vectors of $\tilde{X}_{\mathbf{p}}$, and select the pair $ (\vec{\tilde{x}}_{(1)},\vec{\tilde{x}}_{(2)}) $with the maximal angle between them. All points lie on the plane defined by the three points \( O, \vec{\tilde{x}}_{(1)}, \vec{\tilde{x}}_{(2)} \), and are within the sector formed by the directed rays
\( \overrightarrow{O{\tilde{x}_{(1)}}} \) and \( \overrightarrow{O{\tilde{x}_{(2)}}} \). Then $X\approx \tilde{X}_{\mathbf{p}}=
(\vec{\tilde{x}}_{(1)}\;\;\vec{\tilde{x}}_{(2)})C_{2\times n}$, and $C_{2\times n}$ nonnegative.
\end{proof}

Regardless of whether \( X \) is nonnegative or not, when the rank of \( \tilde{X}_{\mathbf{p}} \) is greater than 2, there exists a semi-NMF decomposition with a nonnegative coefficient matrix, and the algorithm for generating the basis matrix is as follows: we need to select \( k \) point vectors from the column space of \( \tilde{X}_{\mathbf{p}}\) such that they form a basis matrix. The key lies in choosing these \( k \) points so that, together with the origin, they span a high-dimensional \((k+1)\)-polytope that contains all points. The generation algorithm for these \( k \) points is described in section 3. Naturally, in the degenerate cases where the rank of $\tilde{X}_{\mathbf{p}}$ is $1$ or $2$, the algorithm directly recovers the conclusion of the above theorem (which reduces to an NMF decomposition when the input matrix is nonnegative). In more general settings, the algorithm produces a basis matrix that ensures the representation coefficients remain nonnegative.

\section{Algorithm}
\label{sec:alg}
Our analysis leads to the algorithm in the next section. The uniqueness of the basis matrix (and thus the uniqueness of the representation matrix) can be guaranteed if, at each step of rank reduction via oblique projection, the pair of column vectors in $X^{(t)}$ with the maximal distance is unique. Initially, the origin can be included as one of the column vectors in $\tilde{X}_{\mathbf{p}}$ for nonnegative inputting. Once the algorithm outputs the basis matrix $W$, the matrix $\tilde{X}_{\mathbf{p}}$ can be exactly decomposed (or approximated globally optimally for $X$) as $X \approx \tilde{X}_{\mathbf{p}} = WH$, where
\[
H = \left(W^\top W\right)^{-1} W^\top \tilde{X}_{\mathbf{p}}.
\]

\begin{algorithm}
\caption{generating the basis matrix}
\label{alg:buildtree}
\begin{algorithmic}
\STATE \textbf{Input:} A data matrix $X$, and its projection-based zeroing reconstruction $\tilde{X}_{\mathbf{p}} \in \mathbb{R}^{m \times n}$ denoted as $X^{(0)}$, and $\operatorname{rank}(X^{(0)}) = k$
\STATE \textbf{Initialize:}

\IF{$X$ is nonnegative}
\STATE Set $v_0 = 0 \in \mathbb{R}^m$, let $V = \{v_0\}$
\STATE Select the column of $X^{(0)}$ with the largest norm, denote it by $v_{temp}$
\ELSE
\STATE Find the pair of columns in $X^{(0)}$ with the maximum distance
\STATE Let the point with smaller norm be $v_{temp1}$, the larger be $v_{temp2}$
\STATE Let $v_0 = v_{temp1}-0.01f\times(v_{temp2}-v_{temp1})$, and $v_{temp}=v_{temp2}$
\ENDIF

\IF{$k = 1$}
    \STATE let $v_1=v_{temp}$, set $V \gets V \cup \{v_1\}$; \textbf{Output:} $W = v_1 - v_0$; \textbf{terminate}
\ENDIF

\STATE Construct a hyperplane through $v_{temp}$ orthogonal to $v_{temp} - v_0$
\STATE For each column $\vec{\tilde{x}}_i$ of $X^{(0)}$, project the ray from $v_0$ through $\vec{\tilde{x}}_i$ onto the hyperplane
\STATE Replace $X^{(0)}$ with the new projected set, denoted by $X^{(1)}$; note that at this point, the rank of $X^{(1)}$ is $k - 1$

\STATE Set $t \gets 1$

\WHILE{$\operatorname{rank}(X^{(t)}) > 1$}
    \STATE Find the pair of columns in $X^{(t)}$ with the maximum distance
    \STATE Let the point with smaller norm be $v_t$, the larger be $v_{temp}$; update $V \gets V \cup \{v_t\}$
    \STATE Construct a hyperplane through $v_{temp}$ orthogonal to $v_{temp} - v_t$

    \FORALL{remaining columns ${\vec{x}^{(t)}}_i$ (excluding $v_t$) in $X^{(t)}$}
        \STATE Draw the line from $v_t$ to ${\vec{x}^{(t)}}_i$  and compute its intersection with the hyperplane
    \ENDFOR

    \STATE Let new $X^{(t+1)}$ be the set of intersection points (having one fewer column than $X^{(t)}$ )
    \STATE Update $t \gets t + 1$, note that at this point, the rank of $X^{(t)}$ is $k - t$
\ENDWHILE

\STATE  Final step: find the pair of columns in $X^{(t)}$ with the maximum distance; assign the two points with smaller norm to $v_t$, larger to $v_{t+1}$; update $V$
\STATE \textbf{Output:} Basis matrix $W = \begin{bmatrix} v_1 - v_0 & v_2 - v_0 & \cdots & v_k - v_0 \end{bmatrix}$
\end{algorithmic}
\end{algorithm}

\section{Experiments }
\label{sec:experiments}
We compare our result to standard NMF and iterative semi-NMF algorithms in terms of reconstruction error. Notably, our method achieves lower error than standard NMF even on fully nonnegative inputs, and performs comparably or better than iterative semi-NMF algorithms across various datasets. Detailed evidence and experimental validations are provided in the subsequent table.

Table 1 below presents dimensionality reduction and reconstruction experiments conducted on two types of non-negative data: randomly generated synthetic data (4-dimensional, 10 samples) and the UCI Wine dataset (178 samples, 13 dimensions). Three algorithms were applied—standard NMF, the traditional numerically iterative semi-NMF, and our globally optimal semi-NMF—to compare their reconstruction errors. For the synthetic data (4D), the target dimensionalities are 1, 2, and 3. For the UCI Wine data, target dimensions range from 1 to 7. The entry “S:1” in the second row of the first column indicates synthetic data with a target dimensionality of 1, while “Wine:1” in the fifth row corresponds to the UCI Wine dataset with a target dimensionality of 1. The remaining entries follow this pattern accordingly.

\begin{table}[htbp]
\centering
\caption{Comparison of Reconstruction Errors Across Algorithms}
\label{tab:recon_errors}
\begin{tabular}{|c|c|c|c|}
\hline
\backslashbox{Data}{Algorithm} & Standard NMF & Iterative semi-NMF & Our semi-NMF \\
\hline
S:1 &10.610657  &10.610657  & 10.610657 \\
S:2 &4.859226  &4.665386  &  4.653841\\
S:3 &2.206831  &4.748131  &  2.167300\\
Wine:1 &498.520935  &498.520935  &498.520935  \\
Wine:2 &70.140053  &172.181091  &70.140083  \\
Wine:3 &40.666626  &73.077126  &40.664982  \\
Wine:4 &28.695702  &70.092888  &27.342697  \\
Wine:5 &21.401932  &69.924767  &20.094454  \\
Wine:6 &14.900063  &69.265709  &13.950203  \\
Wine:7 &9.568855  &69.638557  &8.533113  \\
\hline
\end{tabular}
\end{table}

From the reconstruction errors in Table 1, it is evident that our globally optimal algorithm not only outperforms the traditional numerically iterative semi-NMF, but also preserves the original data more accurately than the standard NMF.

Next, we examine the distribution of the representation coefficients produced by our algorithm on the UCI Wine dataset. When the target dimensionality is set to 2 or 3, we visualize the coefficient vectors in a 2D plane and a 3D space, respectively. These visualizations not only confirm the non-negativity of the representation coefficients, but also reveal a possible method for visualizing high-dimensional data, offering the most faithful (or realistic) representation.

Figure 1 and 2 show the distributions of the representation vectors after reducing the data to 2 and 3 dimensions, respectively. The original UCI Wine dataset has 13 dimensions and is divided into three classes. After dimensionality reduction to 2D and 3D, the representation coefficients are visualized using different colors and symbols, which provides an intuitive and appropriate depiction of the relationships among the data types.

\begin{figure}[htbp]
  \centering
  \includegraphics{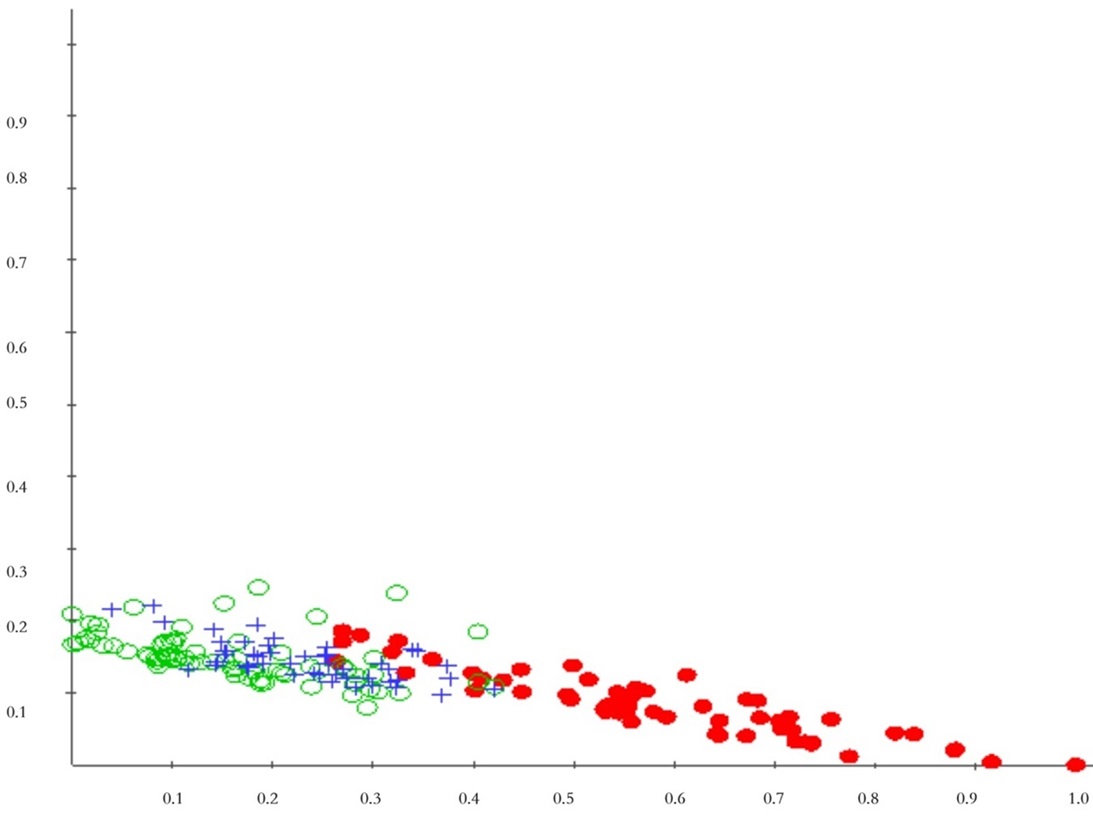}
  \caption{UCI Wine reduced to 2D.}
  \label{fig:testfig}
\end{figure}

\begin{figure}[htbp]
  \centering
  \includegraphics{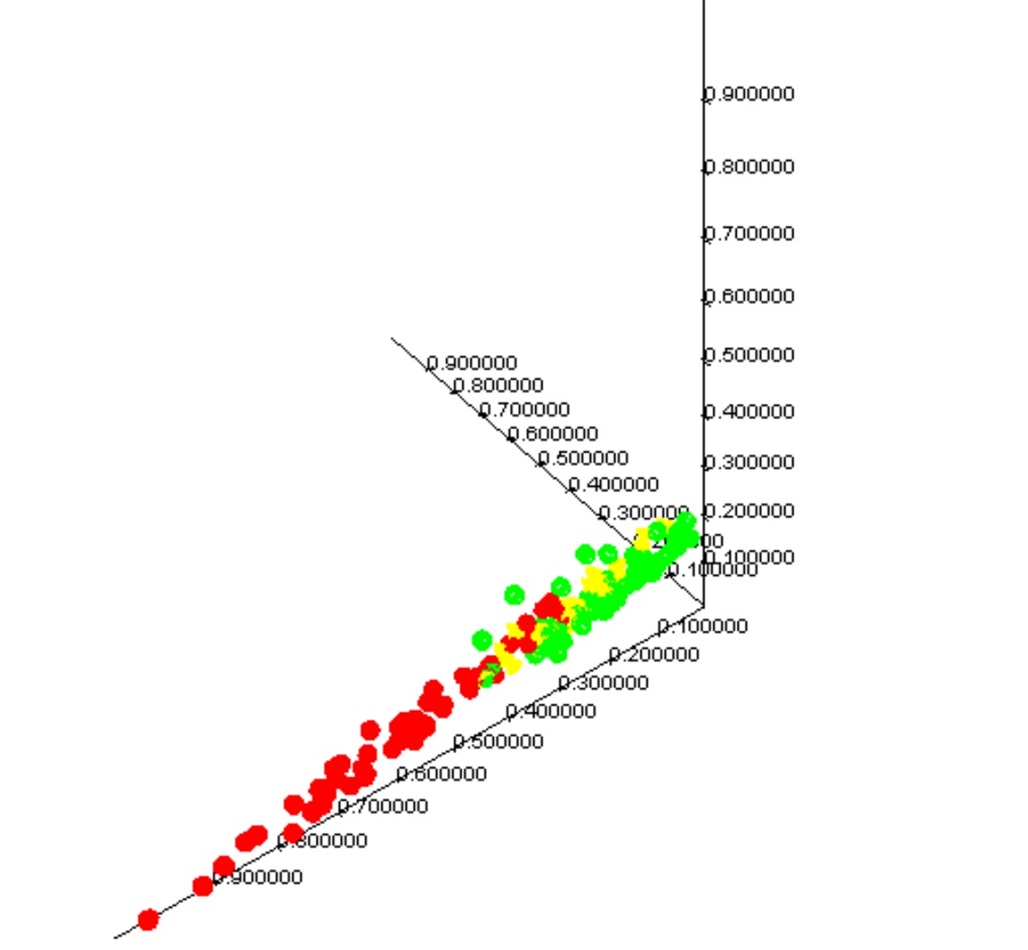}
  \caption{UCI Wine reduced to 3D.}
  \label{fig:testfig2}
\end{figure}





\section{Conclusion}
\label{sec:conclusions}
In this work, we proposed a novel globally optimal approach to semi-NMF by leveraging the orthogonal diagonalization of the scatter matrix of the original data. Our method performs a projection-based zeroing followed by reconstruction, yielding a best approximation of the data in a global sense. Unlike traditional iterative semi-NMF algorithms that may converge to local minima, our approach guarantees a globally optimal factorization under well-defined conditions. Furthermore, the resulting representation matrix is nonnegative, and the reconstruction error is shown to be lower than that of both the standard NMF and conventional semi-NMF methods, particularly for nonnegative input data. For low target ranks (e.g., 1 or 2), our method even recovers fully nonnegative factorizations, aligning with standard NMF outputs. In addition, under the condition that each step in the dimension reduction process preserves the uniqueness of the maximally distant column pairs, we provide a partial guarantee of uniqueness for the resulting matrix factorization.

Future work may explore the robustness of the proposed method in the presence of noise or missing data, as well as its potential integration with large-scale machine learning frameworks. In particular, the deterministic nature and geometric interpretability of the method make it a promising candidate for interpretable representation learning, clustering, and data visualization in high-dimensional settings.

\bibliographystyle{plain}
\bibliography{references} 

\begin{thebibliography}{10}

\bibitem{arora2012}
Sanjeev Arora, Rong Ge, Ravi Kannan, and Ankur Moitra.
\newblock Computing a nonnegative matrix factorization -- provably.
\newblock In {\em Proceedings of the Forty-Fourth Annual ACM Symposium on
  Theory of Computing (STOC)}, pages 145--162, 2012.

\bibitem{bro2003}
Rasmus Bro and Johan A.~K. Suykens.
\newblock A fast non-negativity-constrained least squares algorithm.
\newblock {\em Journal of Chemometrics}, 17:530--541, 2003.

\bibitem{cichocki2009}
Andrzej Cichocki, Rafal Zdunek, Anh~Huy Phan, and Shun-ichi Amari.
\newblock {\em Nonnegative Matrix and Tensor Factorizations: Applications to
  Exploratory Multi-way Data Analysis and Blind Source Separation}.
\newblock Wiley, 2009.

\bibitem{ding2010}
Chris Ding, Tao Li, and Michael~I. Jordan.
\newblock Convex and semi-nonnegative matrix factorizations.
\newblock {\em IEEE Transactions on Pattern Analysis and Machine Intelligence},
  32:45--55, 2010.

\bibitem{Dua:2019}
Dheeru Dua and Efi~Karra Taniskidou.
\newblock {UCI} machine learning repository, 2017.
\newblock \url{https://archive.ics.uci.edu/ml/datasets/wine}.

\bibitem{gillis2020}
Nicolas Gillis.
\newblock {\em Nonnegative Matrix Factorization}.
\newblock SIAM, 2020.

\bibitem{jolliffe2016}
Ian Jolliffe and Jorge Cadima.
\newblock {\em Principal Component Analysis: A Review and Recent Developments},
  volume 379.
\newblock Philosophical Transactions of the Royal Society A, 2016.

\bibitem{jolliffe2002}
Ian~T. Jolliffe.
\newblock {\em Principal Component Analysis}.
\newblock Springer, 2nd edition, 2002.

\bibitem{kuang2015}
Da~Kuang, Chris Ding, and Haesun Park.
\newblock Symmetric nonnegative matrix factorization for graph clustering.
\newblock In {\em Proceedings of the 2015 SIAM International Conference on Data
  Mining}, pages 106--114, 2015.

\bibitem{lee2001}
D.~D. Lee and H.~S. Seung.
\newblock Algorithms for non-negative matrix factorization.
\newblock In {\em Advances in Neural Information Processing Systems 13}, pages
  556--562. MIT Press, 2001.

\bibitem{lin2007}
Chih-Jen Lin.
\newblock Projected gradient methods for nonnegative matrix factorization.
\newblock {\em Neural Computation}, 19(10):2756--2779, 2007.

\bibitem{zdunek2006}
Rafal Zdunek and Andrzej Cichocki.
\newblock Nonnegative matrix factorization with constrained second-order
  optimization.
\newblock In {\em Proc. IEEE International Conference on Acoustics, Speech and
  Signal Processing (ICASSP)}, 2006.

\end{thebibliography}

\end{document}